\documentclass[12pt]{elsarticle}

\usepackage{amssymb}
\usepackage{amsmath}
\usepackage{graphicx}
\usepackage{verbatim}
\usepackage{amssymb}
\usepackage{graphicx}
\usepackage{rotating}
\usepackage{lscape}
\usepackage{subfig}
\usepackage{algorithm}
\usepackage[noend]{algpseudocode}
\usepackage{amsthm}
\usepackage{moresize}

\newtheorem{theorem}{Theorem}
\newtheorem{lemma}[theorem]{Lemma}

\journal{Discrete Optimization}

\begin{document}
	
\setlength{\abovedisplayskip}{3pt}
\setlength{\belowdisplayskip}{3pt}

\begin{frontmatter}



\title{Constraint Programming to Discover One-Flip Local Optima of Quadratic Unconstrained Binary Optimization Problems}

\author{Amit Verma\corref{cor1}}
\ead{averma@missouriwestern.edu}
\cortext[cor1]{Corresponding author}
%\fntext[fn1]{Student}

\author{Mark Lewis}
%\ead{mlewis14@missouriwestern.edu}
%\fntext[fn2]{Lecturer}

\address{Craig School of Business, Missouri Western State University, Saint Joseph, MO, 64507, United States}

\begin{abstract}
	The broad applicability of Quadratic Unconstrained Binary Optimization (QUBO) constitutes a general-purpose modeling framework for combinatorial optimization problems and are a required format for gate array and quantum annealing computers. QUBO annealers as well as other solution approaches benefit from starting with a diverse set of solutions with local optimality an additional benefit. This paper presents a new method for generating a set of one-flip local optima leveraging constraint programming. Further, as demonstrated in experimental testing, analysis of the solution set allows the generation of soft constraints to help guide the optimization process.  
	
	%Quadratic Unconstrained Binary Optimization (QUBO) is a general-purpose modeling framework for combinatorial optimization problems. Recently, the QUBO modeling format has emerged as a requirement for quantum annealers. This paper presents an efficient method based on constraint programming to find local optima for QUBO. Moreover, we devise a hybrid approach that utilizes the information extracted from a set of high-quality locally optimal solutions in the search process. Computational results on benchmark datasets illustrate the efficacy of our method.

%Quadratic Unconstrained Binary Optimization (QUBO) modeling has recently emerged as a unifying framework for solving various optimization problems due to the availability of various generic solvers like tabu search and quantum annealer. Linear equality constraints can be easily incorporated into the objective via penalty functions. However, the penalty coefficient magnitude M increases the complexity of the QUBO model significantly. In this paper, we propose a novel approach to estimate the lower bound of M. We combine this approach with a greedy local search 1-flip solver to illustrate the efficacy of our approach.

\end{abstract}


\end{frontmatter}


\section{Introduction}
%\cite{billionnet2004exact}
%STRICT INEQUALITY - Discuss
The Quadratic Unconstrained Binary Optimization (QUBO) modeling format, $max \; x'Qx; \; x \in \{0,1\}$, has grown in popularity in the last decade and it has been shown that all of Karp's NP-complete problems as well as many constrained problems can be transformed to QUBO (see \cite{kochenberger2014unconstrained} for more details). More recently, QUBO instantiations are a requirement for quantum annealers (\cite{hauke2020perspectives}) which has led to significant interest from the research community. Many QUBO heuristics rely on a starting set of elite solutions (\cite{wang2013backbone,wang2012path,glover2010diversification,samorani2019clustering}) and these starting solutions are key to their performance. 

The set of starting solutions are either generated randomly, or more commonly, through an improvement heuristics such as path relinking, restarts and scatter search (\cite{samorani2019clustering,boros2007local,wang2012path}). This process is limited by the heuristics' ability to find local optima and insure diversity in the elite set. In this paper, we address this shortcoming via a constraint programming (CP) approach for the generation of local optima. Additionally, we present a learning-based method that utilizes the set of local optima to enhance the performance of an existing QUBO solver.

The one-flip local optima $\hat{x}$ for a QUBO has the following characteristics for a maximization problem $max \; x'Qx$:
\begin{align}
& \hat{x}' Q \hat{x} \geq y' Q y \; \forall y \in S_1 (\hat{x}) ; \; \hat{x},y \in \{0,1\} \label{cons}
\end{align}
where $S_1 (\hat{x})$ represents the set of all one-flip neighbors. Hence the solution vectors $y$ and $\hat{x}$ differ by exactly one bit. The total number of such one-flip neighbors is $N$ where $N$ represents the number of variables and $Q$ is an $N$x$N$ matrix of integer or real coefficients. The relationship between $\hat{x}$ and i'th one-flip neighbor $y_i$ is given by $y_i = 1- \hat{x_i}$ and $y_k = \hat{x_k} \; \forall k \in [1,N] : \; k \neq i$. Thus, Equation \ref{cons} leads to $N$ inequalities. $x'Qx$ can be rewritten as $\sum_{i=1}^N (q_{i} x_i + \sum_j q_{ij} x_i x_j)$. We could isolate the impact of flipping the bit corresponding to variable $x_i$ and transform Equation \ref{cons} as:
\begin{align}
& q_i \hat{x_i} + \sum_j q_{ij} \hat{x_i} \hat{x_j} \geq q_i (1-\hat{x_i}) + \sum_j q_{ij} (1-\hat{x_i}) \hat{x_j} \; \forall i \in [1,N]
\end{align}
Note that terms not involving variable $\hat{x_i}$ are eliminated on both sides. Rearranging the terms, we get:
\begin{align}
& 2 q_i \hat{x_i} + 2 \sum_j q_{ij} \hat{x_i} \hat{x_j} \geq q_i + \sum_j q_{ij} \hat{x_j} \; \forall i \in [1,N] \label{final1}
\end{align}

Upon further simplification, the set of equations are reduced to $2 \hat{x_i} \; expr \geq expr$ where $expr =  q_i + \sum_j q_{ij} \hat{x_j}$ and further reduces to $\hat{x_i} > \frac{1}{2}$ if $expr$ is positive and $\hat{x_i} < \frac{1}{2}$ if $expr$ is negative. The following lemma help us in identifying the local optima based on the values of $expr$:
\begin{lemma}
	If $expr < 0$ then $\hat{x_i} = 0$ and if $expr > 0$ then $\hat{x_i} = 1$, else $\hat{x_i}$ can be either $0$ or $1$.
\end{lemma}
The lemma could be enforced by the following set of linear constraints:
\begin{align}
	& q_i + \sum_j q_{ij} \hat{x_j} \leq M \hat{x_i} \; \forall i \in [1,N]\\
	& q_i + \sum_j q_{ij} \hat{x_j} \geq -M(1-\hat{x_i}) \; \forall i \in [1,N]
\end{align}
where $M$ is a large positive number.

A problem instantiated by this model is solved by a CP solver yielding multiple solutions for one-flip local optima. While a similar set of expressions could also be derived for a two-flip local optima (and in general a r-flip local optima), the number of constraints and the associated computational complexity increases significantly. %Also, the strict local maxima requires a minor tweak to the Equation \ref{cons} as $\hat{x}' Q \hat{x} \geq y' Q y + 1$. Herein we assume integral coefficients of the $Q$ matrix and an alternate way of handling $\hat{x}' Q \hat{x} > y' Q y$ which is a required condition for strict optimality. We anticipate to address such topics for future research.

Consider the following $Q$ matrix involving three variables where the coefficients have been doubled and moved to its upper triangular portion:
\[
\begin{bmatrix}
	-4 & 12 & -12\\
	0 & -8 & -8\\
	0 & 0 & 9
\end{bmatrix}
\]
We are interested in obtaining the set of one-flip local optima $\hat{x}$ that satisfies the following constraints based on (\ref{final1}):
\begin{align*}
	-8 \hat{x_1} -12 \hat{x_2} + 12 \hat{x_3} + 24 \hat{x_1} \hat{x_2} -24 \hat{x_1} \hat{x_3} & \geq -4 \\
	-12 \hat{x_1} -16 \hat{x_2} + 8 \hat{x_3} + 24 \hat{x_1} \hat{x_2} - 16 \hat{x_2} \hat{x_3} & \geq -8 \\
	12 \hat{x_1} + 8 \hat{x_2} + 18 \hat{x_3} - 24 \hat{x_1} \hat{x_3} -16 \hat{x_2} \hat{x_3} & \geq 9
\end{align*}
Solving yields a single one-flip local optima $\hat{x}$ given by $[0,0,1]$. Verifying one-flip optimality, the objective function value of $9$ associated with  $[0,0,1]$ is greater than those of the one-flip neighbors $[1,0,1],[0,1,1]$ and $[0,0,0]$ with corresponding objective function evaluations of $-7,-7$ and $0$ respectively.

It is worth noting that all \textit{global} optima are also locally optimal with respect to all possible r-flips with the impact or r-flips being extensively studied. The authors in \cite{alidaee2010theorems} present theoretical formulas based on partial derivatives for quickly determining effects of r-flips on the objective function. \cite{anacleto2020closed} proposed two formulas for quickly evaluating r-flip moves. However, the number of possible r-flip moves to evaluate grows exponentially and one-flip moves are the most commonly implemented approach. %They also present experimental results based on local search and Variable Neighborhood Search. 

%solution or reference solution
Elite sets of high-quality solutions are often used in the design of algorithms for fixing variables. For example, \cite{chardaire1995thermostatistical} fix the variables as the temperature associated with simulated annealing decreases. Their learning process also relies on thresholds and requires some parameter tuning. \cite{wang2011effective} investigated two variable fixing strategies inside their tabu search (TS) routine for QUBO and \cite{zhou2017data} uses a data mining routine to learn frequent patterns from a set of high-quality solutions.% generated by the first phase of GRASP for the Quadratic Assignment Problem. These patterns help in building starting solutions for the optimization heuristics. %The hybrid algorithm performed well with respect to computational efficiency and solution quality. 

%High-quality reference solutions were used to fix or free variables. The authors acknowledge that their fixing routine could lead to errors. However, one of the fixing strategies resulted in the best known solutions for the test problems.

% As opposed to the hard constraint method which may over-zealously eliminate globally optimal solutions, we present a soft constraint methodology using a learning-based approach utilizing the set of one-flip locally optimal solutions.
Fixing/freeing variables have also been explored in the context of a quantum annealer by \cite{karimi2017boosting}. The authors reduce the QUBO by fixing some variables to values that have a high probability of occurrence in the sample set of solutions. In contrast to fixing, some of the approaches learn to avoid local optima in the search process. \cite{basharu2007escaping} compare two strategies of escaping local optima: (a) assigning penalties to violated constraints (b) assign penalties to individual variable values participating in a constraint violation. Their results quantify the impact of penalties on the solution landscape.% of the transformed problem caused by penalties. 

%or the features of select solution vectors 
%The latter phase either fixes or frees selected variables.
The reference/elite set and other problem features are also useful in designing various metaheuristics. For example, \cite{wang2013backbone} apply backbone guided TS to QUBO alternating between a TS phase and a phase that fixes/frees strongly determined variables. \cite{voudouris2003guided} and \cite{whittley2004attribute} utilize problem features to guide the local search routine where the objective function is augmented with penalty terms based on problem features. 

 %Note that they also adopt a hard constraint approach, which might lead to sub-optimal solutions. %More recently, the same set of authors showed in \cite{karimi2017effective} that their technique is effective for a range of solvers and difficult optimization problems. 

%Our reformulation based on the transformed $Q$ matrix could be used to improve the performance of existing QUBO solvers (like \cite{verma2020penalty} and \cite{verma2020optimal}).

%DISCUSS THE IMPORTANCE OF FINDING HIGH QUALITY SAMPLES.

% WHERE TO INCLUDE THE FOLLOWING
%Note that we present a computationally efficient method to obtain a set of local optima and address the challenges raised in \cite{fieldsend2018computationally}.
The contribution of our paper is twofold. First, we present a new constraint programming approach to obtain a set of one-flip local optima for QUBO. These high-quality samples can be used as a starting elite solutions set and can also be utilized for the construction of Local Optima Networks (see \cite{ochoa2014local} for more details) for a wide variety of combinatorial problems that fit into the QUBO framework. Second, we provide an approach to utilize the information contained in the set of local optima through penalties and rewards by transforming the $Q$ matrix using two variants that favor or avoid the set $L$ of locally optimal solutions. Our reformulations could be used to improve the performance of existing QUBO solvers (like \cite{verma2020penalty} and \cite{verma2020optimal}). %Note that we generate the set $L$ a priori. However, more sophisticated techniques could be developed wherein the one-flip locally optimal elite set is dynamically updated  and integrated into an ongoing search. We reserve such topics for future research.%and two-flip

%STRICT LOCAL OPTIMA MIGHT LEAD TO FEWER DEGENERATE SOLUTIONS

%TODO - Discuss Strict Local Optima, Literature Review (\cite{alidaee2010theorems}, \cite{anacleto2020closed}, \cite{verma2020penalty}, \cite{glover2018logical})
%DISCUSS COMPACT FORMULATION FOR $lb_e$ and $ub_e$
%COMMENT ON COMPLEXITY
%CAN WE RELATED one-flip LOCAL OPTIMA WITH two-flip LOCAL OPTIMA
\section{Learning Approach}

There are various ways to utilize the information provided by the set of local optima $L$. Note that $L$ could be obtained by satisfying the constraints corresponding to one-flip local optima detailed in Section 1. Herein we present a simple approach that relies on the number of times a specific variable $x_i$ is set to $0$ or $1$. If the variable $x_i$ takes a specific value more frequently in the set of local optima, there are two schools of thought in the literature to handle it. First, favor local optima and hypothesize that there is a high chance that the global optima would also have such a variable $x_i$ set to $0$ or $1$ respectively. Second, design heuristics to avoid the set of local optima and aid the solver to explore new areas in the solution landscape while avoiding local optima. Our approach to calculate the frequency of occurrence for each variable $x_i$ in the set of local optima is outlined in Algorithm 1.%In this way, progress is made towards global optima based on the properties of the set of local optima $L$ implemented as soft constraints incorporated into $Q$. 

\begin{algorithm}
	\scriptsize%\footnotesize
	\caption{Frequency calculation based on the set of local optima}
	\label{learning}
	\begin{algorithmic}[1] % The number tells where the line numbering should start
		\Procedure{Frequency}{$L$} \Comment{Returns the relative frequency of setting $x_i = 0/1$ in $L$}
		\State $freq_0 \gets 0$
		\State $freq_1 \gets 0$
		\For{$i = [1,N]$} \Comment{For all variables}
		\For{$k = [1,|L|]$} \Comment{For all locally optimal solutions}
		\State $If \; x[i] == 0, freq_0[i] \gets freq_0[i] + 1$
		\State $If \; x[i] == 1, freq_1[i] \gets freq_1[i] + 1$
		\EndFor
		\EndFor
		\State $freq_0 \gets freq_0/|L|$
		\State $freq_1 \gets freq_1/|L|$
		\State \textbf{return} $freq_0$ and $freq_1$\Comment{The chance of setting a variable $x_i$ to $0/1$}
		\EndProcedure
	\end{algorithmic}
\end{algorithm}

%Note that we extract the number of times a variable $x_i=0/1$ in the set of local optima $L$ and store it in an array entry $freq_0[i]$ and $freq_1[i]$ respectively.

At the end of this process, we return the relative frequency by dividing each $freq$ entry with the number of locally optimal solutions $|L|$. We use this information contained in $freq$ in multiple ways. Noting that $freq$ is the chance of setting a specific variable $x_i$ to $0/1$ in the set $L$, then if the value of $freq_1[i]$ is close to $1$, the variable $x_i$ is set to $1$ in majority of the locally optimal solutions. Thus, we consider setting a variable $x_i = 1$ if $freq_1[i] >= \alpha$ where $\alpha$ is a user-defined parameter. On the other hand, the solver should escape the locally optimal solutions by disincentivizing $x_i = 1$ if $freq_1[i] >= \alpha$ so that the solver avoids replicating the behavior observed in the set of local optima. We explore both variants of the transformation approach designed to (i) favor local optima (ii) escape local optima. For this purpose, we will adjust the linear coefficients of the original $Q$ matrix to generate $Q_1$ and $Q_2$ for the two strategies with the typical values of $\alpha$ ranging from $95-100\%$. % Note that the threshold parameter $\alpha$ could vary based on the observed number of local optima.

%Although we could consider the information extracted from the set of local optima as a hard constraint, that could create a new solution space that excludes global optima. Since we are interested in aiding the solver to yield improved solutions, we employ soft constraints.

%The superscripts $l$ and $q$ identify strategies that make updates to only the linear terms and only the quadratic terms respectively. 

The technique for generating the two transformed matrices $Q_1$ and $Q_2$ based on strategies of favoring and escaping local optima are implemented as soft constraints and summarized in Algorithm 2. Specifically, if the chance of a variable $x_i$ to be set to $1$ (given by $freq_1[i]$) is greater than or equal to $\alpha$, we add a reward $\delta$ to the linear coefficient $q_i$ for favoring local optima. Thus, for strategy (i), we use the transformed $Q_1$ matrix involving $q_i^{1} \leftarrow q_i^1 + \delta$. This change incentivizes any solver to set $x_i=1$. Similarly, we make updates to every linear coefficient in the transformed matrix whenever $freq_1[i] \geq \alpha$. A large value of $\delta$ enforces the constraint $x_i = 1$ strictly. However, it could also alter the solution landscape for the solver. Conversely, a penalty term $-\delta$ added to the linear coefficient $q_i$ is utilized as a proxy for the constraint $x_i = 0$ in a maximization problem (if $freq_{0}[i] \geq \alpha$). The changes are reversed for the second strategy of avoiding local optima. %In the future we will investigate augmenting quadratic terms. %We could also augment each quadratic coefficient involving the variable pair $x_i$ and $x_j$ depending on the variable assignments, but we reserve such topics for future research.

\begin{algorithm}
	\scriptsize%\footnotesize
	\caption{Transformation Approach}
	\label{transform}
	\begin{algorithmic}[1] % The number tells where the line numbering should start
		\Procedure{Transformation}{$Q,freq,\alpha,\delta$} \Comment{Returns the transformed matrices $Q_1$ and $Q_2$}
		\State $Q_1 \gets Q$
		\State $Q_2 \gets Q$
		\For{$i = [1,N]$} \Comment{For all variables}
		%\For{$k = [1,|L|]$} \Comment{For all local optimal solutions}
		\State $If \; freq_0[i] \geq \alpha, Q_1[i,i] \gets Q_1[i,i] - \delta$
		\State $If \; freq_1[i] \geq \alpha, Q_1[i,i] \gets Q_1[i,i] + \delta$
		\State $If \; freq_0[i] \geq \alpha, Q_2[i,i] \gets Q_2[i,i] + \delta$
		\State $If \; freq_1[i] \geq \alpha, Q_2[i,i] \gets Q_2[i,i] - \delta$
		%\EndFor
		\EndFor
		%\State $freq_0 \gets freq_0/|L|$
		%\State $freq_1 \gets freq_1/|L|$
		\State \textbf{return} $Q_1$ and $Q_2$ \Comment{The transformed matrices based on strategies (i) and (ii)}
		\EndProcedure
	\end{algorithmic}
\end{algorithm}

%Moreover, if the cumulative frequency based on the observed local optimal solutions is greater than $beta$, then the resulting assignment is $q_i \leftarrow q_i + \delta$.  
%We perform a sensitivity analysis on the $\alpha$ parameter in Section \ref{expt}. 

%Again, the choice of $\delta$ would influence the performance of the solver. We would study the impact of the parameters $\alpha$ and $\delta$ on the computational experiments in Section \ref{expt}.

\section{Computational Experiments} \label{expt}
%designated by $nodes_edges_instance$ and $bqp_nodes_instance$ respectively
For testing we use the QUBO instances presented in \cite{glover2018logical} and \cite{beasley1990or}. The algorithms were implemented in Python 3.6. The experiments were performed on a 3.40 GHz Intel Core i7 processor with 16 GB RAM running 64 bit Windows 7 OS. The datasets described in \cite{glover2018logical} have $1000$ nodes while the ORLIB instances \cite{beasley1990or} have $1000$ and $2500$ nodes. Our experiments utilize a path relinking and tabu search based QUBO solver.  %The datasets have $2500$ variables. %We will evaluate the performance of both quadratic and linear models detailed in Section 1. We also utilize the linear CP SAT solver provided by Google. 

A one-flip tabu search with path relinking was modified from (\cite{rna}). The primary power of a one-flip search is its ability to quickly evaluate the effect of flipping a single bit,  $x_i \rightarrow 1- x_i$, allowing selection of the variable having the greatest effect on a local solution in $O(n)$ time (\cite{kochenberger2004unified}) as opposed to directly evaluating $x'Qx$ which is $O(n^2)$.  The search used in this paper accepts an input $Q$ matrix as well as a starting elite set of solutions of size $S$.  It performs path relinking between the solutions in $S$ to derive a starting solution where path relinking is implemented as a greedy search of the restricted solution space defined by the difference bits of a solution pair. The relinking generates a starting solution from which a greedy search is performed by repeatedly selecting the single non-tabu variable that has the largest positive impact on the current solution. Variables selected to be flipped are given a tabu tenure to avoid cycling.  When there are no non-tabu variables available to improve the current solution then a backtracking operation is performed to undo previous flips.  When no variable (tabu or not) is available to improve the current solution then a local optimum has been encountered and backtracking is performed.

The diversity attributes are measured as the mean hamming distance between all pairs of solution vectors. Note that the hamming distance $d(a,b)$ between two binary vectors $a$ and $b$ is given by the number of difference bits. The mean hamming distance $\mu_d$ is given by $\sum_{a,b \in L}^{a \neq b} d(a,b)/|L|$. Similarly, we can measure the quality of the elite set by the mean objective function, $\mu_{Obj}$. For benchmarking, we utilize a common approach  presented in the literature (\cite{wang2013backbone,wang2012path,glover2010diversification,samorani2019clustering}) to generate elite sets consisting of a randomized solution improved by a greedy heuristic until a local one-flip optima is reached and the solution added to the elite set and the process repeated. For both the CP solver and the greedy heuristic, we allot $600$ seconds and extract the top $500$ local optima sorted by the objective function to favor high-quality solutions.

% Table generated by Excel2LaTeX from sheet 'Present_Diversity'
\begin{table}[htbp]
	\centering
	\caption{Diversity Attributes of CP Approach}
	\scalebox{0.67}{
		\begin{tabular}{r|r|r|r|r}
			\hline
			\multicolumn{1}{l}{bqp2500} & \multicolumn{2}{c}{CP Solver} & \multicolumn{2}{c}{Greedy Search Heuristic} \\
			\multicolumn{1}{l}{Instance} & \multicolumn{1}{l}{$\mu_d$} & \multicolumn{1}{l}{$\mu_{Obj}$} & \multicolumn{1}{l}{$\mu_d$} & \multicolumn{1}{l}{$\mu_{Obj}$} \\
			\hline
			1     & 521.5 & 996297.9 & 266.0 & 1504173.5 \\
			2     & 500.0 & 1008336.9 & 286.5 & 1460415.7 \\
			3     & 535.3 & 941038.3 & 266.6 & 1403657.1 \\
			4     & 547.8 & 979253.4 & 221.0 & 1499321.1 \\
			5     & 510.7 & 1009030.4 & 240.4 & 1481973.9 \\
			6     & 473.4 & 991556.1 & 237.2 & 1460578.2 \\
			7     & 492.8 & 993305.4 & 292.7 & 1467269.2 \\
			8     & 569.5 & 969050.3 & 216.5 & 1476406.9 \\
			9     & 458.0 & 1014991.4 & 251.5 & 1472053.8 \\
			10    & 462.7 & 1002726.8 & 295.2 & 1470710.0 \\
			\hline
	\end{tabular}}%
	\label{tab:tab1}%
\end{table}%

The results are presented in Table 1. The CP approach leads to more diverse solutions since $\mu_d$ for the CP solver is almost double those of the greedy heuristic. While the greedy approach obtains solutions with higher objective values, they are less diverse, hence the CP approach provides a compromise between solution diversity and the objective value. %In addition, the CP approach is fast and obtains thrice as many local optima in the limited time frame than the greedy heuristic.

%We experiment with two types of solver: (a) CPLEX solver with $qtolin = 0$ (behaves like an exact quadratic solver) (b) Path Relinking solver based on CITE.

%have [50,100,250,500,1000]
%TODO - Present Quadratic vs Linear CP Model for IBM CPLEX and Google CP-SAT; Present hard constraint vs soft constraint

%\section{Results} \label{results}

%TODO - FORMAT OF THE TABLES 

To assess the impact of different soft constraint thresholds, we experiment with the following values of $\alpha$: (a) $0.99$ (b) $0.975$ (c) $0.95$ and the following settings of $\delta$: (a) $2\%$ (b) $5\%$ (c) $10\%$ wherein the percentage is expressed in terms of the maximum value of the coefficients of the $Q$ matrix. For example, the coefficients of the ORLIB datasets lie in the range $[-100,100]$. Thus, the linear coefficients of the $Q$ matrix are adjusted by $\delta = 2, 5$ or $10$ units.

For the nine different parameter combinations of $(\alpha,\delta)$, we allotted $100$ seconds each. For a fairer comparison, the benchmark experiments on the original $Q$ matrix are run for a total of $900$ seconds. We present two different versions of our heuristic based on $Q_1$ and $Q_2$ in Table 2. The columns ``$Obj_Q$", ``$Improv_{Q1}$" (and ``$Improv_{Q2}$") represent the best objective function obtained by the QUBO solver using the $Q$ matrix within $900$ seconds and the percentage improvement in the best objective function among nine parameter combinations of $(\alpha,\delta)$ using $Q_1$ (and $Q_2$) matrix respectively.% (each run on the transformed matrices is allotted $100$ sec).

% Table generated by Excel2LaTeX from sheet 'Final-Present_Deltav3'
\begin{table}[htbp]
	\centering
	\caption{Results of Algorithm 2}
	\scalebox{0.67}{
		\begin{tabular}{|lrrr|lrrr|}
			\hline
			Instance & \multicolumn{1}{l}{$Obj_Q$} & \multicolumn{1}{l}{$Improv_{Q1}$} & \multicolumn{1}{l}{$Improv_{Q2}$} & Instance & \multicolumn{1}{l}{$Obj_Q$} & \multicolumn{1}{l}{$Improv_{Q1}$} & \multicolumn{1}{l}{$Improv_{Q2}$} \\
			\hline
			1000\_5000\_1 & 25934 & 0.38  & 0.28  & 1000\_10000\_1 & 42920 & 0.91  & 1.35 \\
			1000\_5000\_2 & 483289 & 0.03  & 0.06  & 1000\_10000\_2 & 893493 & 0.59  & 0.63 \\
			1000\_5000\_3 & 52469 & 0.03  & 0.02  & 1000\_10000\_3 & 96764 & 0     & -0.01 \\
			1000\_5000\_4 & 214726 & -0.39 & 0.24  & 1000\_10000\_4 & 371621 & 1.66  & 0.9 \\
			1000\_5000\_5 & 18644 & 0.01  & -0.02 & 1000\_10000\_5 & 29870 & -0.03 & -0.01 \\
			1000\_5000\_6 & 275332 & 0.2   & 0.37  & 1000\_10000\_6 & 476253 & 0.31  & 0.26 \\
			1000\_5000\_7 & 32141 & 0.26  & 0.31  & 1000\_10000\_7 & 55732 & 0.17  & 0.19 \\
			1000\_5000\_8 & 155738 & 0.12  & -0.13 & 1000\_10000\_8 & 250964 & 0.17  & -0.02 \\
			1000\_5000\_9 & 270749 & 0.49  & 0.25  & 1000\_10000\_9 & 479986 & -0.17 & 0.14 \\
			1000\_5000\_10 & 18385 & 0.05  & -0.04 & 1000\_10000\_10 & 29624 & 0     & 0.02 \\
			1000\_5000\_11 & 158718 & 0.02  & 0.07  & 1000\_10000\_11 & 255999 & 0.83  & 0.96 \\
			1000\_5000\_12 & 32297 & 0.01  & 0.01  & 1000\_10000\_12 & 54825 & 0.01  & 0.01 \\
			1000\_5000\_13 & 477743 & 0.54  & 0.08  & 1000\_10000\_13 & 870231 & 0.04  & 0.26 \\
			1000\_5000\_14 & 25848 & 0.04  & 0.01  & 1000\_10000\_14 & 43236 & 0.12  & 0.01 \\
			1000\_5000\_15 & 214435 & -0.01 & 0.56  & 1000\_10000\_15 & 374992 & 0.62  & 0.85 \\
			1000\_5000\_16 & 52686 & 0     & -0.01 & 1000\_10000\_16 & 97105 & 0.02  & 0.19 \\
			\hline
			bqp\_1000\_1 & 371155 & 0.07  & 0.07  & bqp\_2500\_1 & 1512444 & 0.23  & 0.18 \\
			bqp\_1000\_2 & 354822 & 0.02  & 0.03  & bqp\_2500\_2 & 1469553 & 0.03  & 0.09 \\
			bqp\_1000\_3 & 371236 & 0     & 0     & bqp\_2500\_3 & 1413186 & 0.04  & 0.04 \\
			bqp\_1000\_4 & 370638 & -0.01 & 0     & bqp\_2500\_4 & 1506521 & 0.07  & 0.07 \\
			bqp\_1000\_5 & 352730 & 0     & 0     & bqp\_2500\_5 & 1491700 & 0.01  & -0.01 \\
			bqp\_1000\_6 & 359629 & 0     & 0     & bqp\_2500\_6 & 1468745 & -0.01 & -0.05 \\
			bqp\_1000\_7 & 370718 & 0.13  & 0.11  & bqp\_2500\_7 & 1478073 & 0.02  & -0.03 \\
			bqp\_1000\_8 & 351975 & 0     & 0.01  & bqp\_2500\_8 & 1483757 & 0.03  & 0.02 \\
			bqp\_1000\_9 & 349044 & 0.06  & 0.08  & bqp\_2500\_9 & 1482091 & 0.01  & 0.02 \\
			bqp\_1000\_10 & 351272 & 0.04  & -0.01 & bqp\_2500\_10 & 1482220 & -0.02 & 0 \\
			\hline
	\end{tabular}}%
	\label{tab:addlabel}%
\end{table}%

In summary, favoring or escaping local optima based on $Q_1$ and $Q_2$ leads to improvement in solution quality in the majority of the instances. Moreover, utilizing both $Q_1$ and $Q_2$ results (i.e. looking at $max(Improv_{Q1},Improv_{Q2})$) leads to a guaranteed improvement in all but two instances (1000\_10000\_5 and bqp\_2500\_6). Future research will explore this dynamic through a parallelized tabu search with alternating phases between $Q, Q_1$ and $Q_2$.

We conducted a paired two-sample t-test between ``$Improv_{Q1}$ and ``$Improv_{Q2}$" columns to determine whether the population mean of the $Q_1$ results was different that that of $Q_2$ and found there is no statistically significant difference between the two techniques. Moreover, no specific combination of $(\alpha,\delta)$ was dominant over all others.% always dominated improvements for all datasets of a given type.

%Next, we analyze the number of variables set to $0/1$ according to our CP approach. For this purpose, we present the results for bqp\_2500\_1.txt in Figure 1. Note that we visualize the relative value of $freq_1$ using $500$ solutions in the elite set. According to Algorithm 2, $freq_1 \geq \alpha$ for a given variable leads to nudges in the linear coefficient captured by transformations $Q_1$ and $Q_2$ respectively. Moreover, for a given variable $freq_1$ and $freq_0$ are complements connected by the equation $freq_1 + freq_0 = 1$. We can clearly observe from the figure that around $600$ and $400$ variables' linear coefficients are impacted based on $freq_1$ and $freq_0$ values with $\alpha=0.95$. Thus, almost 40\% (1000 out of 2500) variable coefficients are updated according to Algorithm 2. Future research would explore the ties between Algorithm 2 and the QUBO preprocessor detailed in \cite{glover2018logical}.

\begin{comment}
\begin{figure}[htbp]
	\centerline{\includegraphics[scale=0.9]{Vars_Freq_1.png}}
	\caption{Number of Variables vs Relative Frequency of setting $x_i = 1$}
	\label{fig:fig11}
\end{figure}
\end{comment}

\section{Conclusions} \label{conc}
We present a Constraint Programming approach to obtain a diverse set of local optima which could be utilized in the elite sets or local optima networks, and we present a learning-based technique that modifies the linear coefficients of the $Q$ matrix while favoring or avoiding local optima. Testing indicates this technique leads to improvement in solution quality for benchmark QUBO instances. Future work involves combining the effects of $Q_1$ and $Q_2$ in an alternating phase tabu search heuristic.

\bibliographystyle{elsarticle-num} 
%\bibliography{rFilp}
{\footnotesize%\scriptsize%\small
\bibliography{rFilp}}


%{\footnotesize \bibliography{qsubs} }  % name your BibTeX data base
%% \bibliographystyle{elsarticle-harv} 

%\begin{thebibliography}{00}


%\bibitem{}

%\end{thebibliography}
\end{document}